\documentclass[letterpaper]{article}

\usepackage[]{aaai2027}  
\usepackage[hyphens]{url}  
\usepackage{graphicx} 
\usepackage{natbib}  
\usepackage{caption} 
\usepackage{amsmath,amssymb,amsfonts}
\usepackage{amsthm}
\usepackage{bm}
\usepackage{booktabs}
\usepackage{multirow}
\usepackage{xcolor}
\usepackage{tikz}
\usepackage{subcaption}
\usepackage[ruled,vlined]{algorithm2e}

\definecolor{beforegray}{HTML}{8A8A8A}
\definecolor{highblue}{HTML}{0072B2}
\definecolor{lowgreen}{HTML}{009E73}
\definecolor{markerred}{HTML}{D55E00}

\definecolor{beforegray}{RGB}{165,165,165}
\definecolor{highblue}{RGB}{20,30,185}
\definecolor{lowgreen}{RGB}{20,130,95}
\definecolor{markerred}{RGB}{190,20,20}

\newcommand{\sg}{\operatorname{sg}}
\newcommand{\I}{\mathbb{I}}

\newcommand{\Vocab}{\mathcal{V}}

\definecolor{gaincolor}{RGB}{190,45,45}
\definecolor{dropcolor}{RGB}{0,120,80}

\newcommand{\gain}[1]{\ensuremath{+#1}}
\newcommand{\drop}[1]{\ensuremath{-#1}}

\newtheorem{theorem}{Theorem}

\newtheorem{proposition}{Proposition}

\title{Influence-Directed Distillation: Solving the Diversity Bottleneck in Sampled-Token On-Policy Distillation}

\author {
    Run Yang\textsuperscript{\rm 1},
    Runpeng Dai\textsuperscript{\rm 2},
    Jie Sun\textsuperscript{\rm 3},
    Jielei Zhang\textsuperscript{\rm 1},
    Fan Zhou \textsuperscript{\rm 4},
    Hongtu Zhu\textsuperscript{\rm 2},
    Peiyi Li\textsuperscript{\rm 1},
    Longwen Gao\textsuperscript{\rm 1}
}
\affiliations {
    \textsuperscript{\rm 1}BiliBili.Inc\\
    \textsuperscript{\rm 2}University of North Carolina at Chapel Hill\\
    \textsuperscript{\rm 3}University of Science and Technology of China\\
    \textsuperscript{\rm 4}Shanghai University of Finance and Economics \\
}

\begin{document}

\maketitle
\begin{abstract}
Sampled-token on-policy distillation (OPD) efficiently transfers capabilities from teacher to student using student-generated tokens, requiring teacher probabilities only for sampled tokens. Yet it frequently suffers from diversity distillation failure: the student's pass@1 improves while its pass@$k$ plateaus, failing to inherit the teacher's diversity. To explain this, we introduce \emph{First-Order Local Entropy Influence}, a signed first-order proxy that decouples each update's entropy effect into the teacher--student log-probability gap and the student's local probability structure, and empirically links entropy contraction to negative-influence positions. Motivated by this, we propose \emph{Influence-Directed Adaptive On-Policy Distillation} (IDA-OPD): rather than relying on costly full-vocabulary Forward-KL objectives, it preserves entropy-expanding updates while replacing entropy-contracting ones with divergence-adaptive advantage shrinkage, using only the teacher's sampled-token log-probability. Experiments on reasoning-oriented distillation show IDA-OPD consistently improves pass@$k$, inheriting the teacher's diversity through distillation, matches the strongest teacher-informed methods at strictly lower cost, and broadly maintains vanilla OPD's pass@1, all without full-vocabulary teacher information.
\end{abstract}

\section{Introduction}


On-policy distillation (OPD) is increasingly central to large language model (LLM) post-training.
It has been widely integrated into the post-training pipelines of recent large-scale systems, including GLM-5 (Zeng et al. 2026), KiMi (Team et al. 2025), and Qwen3 (Yang et al. 2025), yielding significant improvements. Its standard form, however, distills against the teacher's probability distribution over the full vocabulary, which is costly to compute and store at scale; this motivates sampled-token OPD, which needs only the teacher's sampled-token probability and avoids full-vocabulary logits ~\citep{fu2026revisiting,jia2026aopd}.

However, sampled-token OPD frequently suffers \emph{diversity distillation failure}: the student's pass@1 improves while its pass@$k$ plateaus~\citep{fu2026revisiting,cui2025entropy}.



Existing methods for protecting diversity fall into two categories, which trade off precision and efficiency. Teacher-informed methods inject richer teacher signals to guide the student: divergence-based criteria target areas of high teacher-student disagreement~\citep{xu2026tip,wang2026teachability}, AOPD decouples exploration from imitation in these specific regions~\citep{jia2026aopd}, and Entropy-Aware OPD introduces a forward-KL penalty for high-uncertainty tokens~\citep{jin2026entropy}. While effective, they forfeit the computational efficiency of sampled-token OPD by reintroducing expensive top-$K$ teacher distributions via Forward KL. In contrast, {student-side} methods remain computationally efficient by directly inflating entropy. They achieve this by relaxing heavy-tailed credits~\citep{ko2026reopold}, reinforcing negative samples~\citep{zhu2026surprising}, adding an entropy bonus~\citep{schulman2017proximal}, or adapting advantage shapes and entropy coefficients~\citep{cheng2025reasoning,zhang2025revisiting,jiang2025rethinking,yu2025dapo}. However, their intervention is blunt; they raise entropy globally without pinpointing the specific updates that drain diversity, ultimately leaving much of the performance gap unresolved. This raises a question:

\begin{quote}
\emph{Can we protect generation diversity using
only the teacher's sampled-token probability?}
\end{quote}

Following prior work that studies generation diversity through the lens of policy entropy~\citep{wang2025beyond,cui2025entropy,dai2026cde}, we study how updates in sampled-token OPD affect the student's entropy change.
We find that this entropy change can be modeled by the \emph{First-Order Local Entropy Influence} $\mathcal{I}_H(y)$, which mathematically decouples the entropy shift into two components: the teacher--student log-probability difference, i.e., the advantage $A_y$, and the student's local probability structure.
Since the local probability structure is fixed at a given position, controlling $A_y$ therefore controls the update's entropy change.

However, uniformly shrinking $A_y$ at all entropy-contracting positions would blindly discard essential learning signals.
This motivates a closer look at where the contraction originates to design a selective intervention.
We measure the cumulative entropy loss across positions, and Figure~\ref{fig:entropy_flux} shows the contraction peaks sharply in the low-discrepancy region.
Because their teacher--student discrepancy is small, forcing these updates yields little corrective signal per token yet, in aggregate, continuously drains the model's entropy.

Combining these two findings yields our method, \emph{Influence-Directed Adaptive On-Policy Distillation} (IDA-OPD).
We keep entropy-expanding updates intact and reweight the advantage $A_y$ of each entropy-contracting update ($\mathcal{I}_H(y)<0$) by a divergence-adaptive weight $w_y$: the attenuation is strongest where the teacher--student discrepancy is small, exactly the low-value updates that dominate the entropy drain, while high-discrepancy corrections are left intact.

Extensive experiments on reasoning-oriented distillation demonstrate
that IDA-OPD substantially mitigates diversity distillation failure. It
curbs the entropy contraction and consistently improves pass@$k$ over methods that operate under the same
sampled-token budget, while remaining on par with the
strongest teacher-informed methods at a strictly lower information
cost---broadly maintaining pass@1 accuracy and using no full-vocabulary
teacher information. Our analysis
further shows that selecting updates by the sign of $\mathcal{I}_H(y)$
under real training controls the student's entropy trajectory, and that
our divergence-adaptive shrinkage targets the low-discrepancy region
where entropy contraction concentrates.

Our core contributions are summarized as follows:

%
%
%
\begin{itemize}
\item We trace the diversity distillation failure of sampled-token OPD, the gap between improved pass@1 and stagnant pass@$k$, to entropy-contracting updates, including a substantial share at low-divergence positions.

\item We formulate \emph{First-Order Local Entropy Influence}, a signed metric decoupling the teacher--student log-probability gap $A_y$ from the student's local probability structure $D_y$, whose sign controls the entropy trajectory under real training.

\item We propose IDA-OPD, which reweights entropy-contracting advantages with a divergence-adaptive shrinkage weight $w_y$, improving diversity over same-budget methods and matching the strongest teacher-informed ones without full-vocabulary logits or auxiliary FKL objectives.
\end{itemize}

\section{Preliminaries}

We formalize the training process of on-policy distillation as follows. Let $\Vocab$ denote the vocabulary and let $h$ be a decoding prefix visited by a student rollout, with student and teacher next-token distributions $p_h(\cdot)=p_\theta(\cdot\mid h)$ and $q_h(\cdot)=p_{\mathrm T}(\cdot\mid h)$. OPD matches these distributions on student-induced prefixes by optimizing the reverse KL $\mathbb E_{h\sim p_\theta}[D_{\mathrm{KL}}(p_h\|q_h)]$.

Computing this divergence exactly requires full-vocabulary teacher logits at every prefix, which is prohibitively expensive. Since reverse KL is an expectation under the student distribution, sampled-token OPD or K1 estimator is often applied by drawing a single token from distribution $y\sim p_\theta(\cdot\mid h)$, which queries only the teacher log-probability of $y$ \cite{fu2026revisiting,jia2026aopd}. The advantage of the sampled token is
\begin{equation*}
    A_y
    =
    \log p_{\mathrm T}(y\mid h)-\log p_\theta(y\mid h),
\end{equation*}
and the vanilla sampled-token OPD loss is
\begin{equation*}
    \ell_y^{\mathrm{OPD}}
    =
    -\sg(A_y)\log p_\theta(y\mid h),
    \label{eq:vanilla_opd}
\end{equation*}
where $\sg(\cdot)$ is the stop-gradient operator. Although locally unbiased, this efficient estimator replaces the full teacher distribution with a noisy one-token signal, which can reduce diversity and cause premature entropy collapse \cite{jin2026entropy,ko2026reopold}. Our work therefore aims to preserve student diversity while retaining the efficiency of the sampled-token setup.

\section{Understanding Diversity Distillation Failure}
\label{sec:understanding_failure}

In this section, we investigate the mechanism underlying diversity distillation failure in sampled-token OPD through the lens of entropy. We first show that the advantage $A_y$ alone does not determine the entropy effect of an update, which also depends critically on the student's current distribution. This observation motivates a first-order analysis, through which we derive First-Order Local Entropy Influence to characterize the entropy effect of each sampled-token update.

\subsection{Advantage Alone Does Not Determine Entropy}
\label{subsec:advantage_entropy}
In sampled-token OPD, $A_y$ governs the update to the sampled token: its sign determines whether $y$ is reinforced or suppressed, while its magnitude controls the update strength. Accordingly, prior work has used $A_y$ as a natural signal for analyzing and regulating OPD~\citep{ko2026reopold,jia2026aopd}.

However, our analysis reveals that $A_y$ alone is insufficient to determine the effect on entropy. Figure~\ref{fig:advantage_entropy_intuition}(\subref{fig:advantage_vs_entropy}) illustrates this empirically by plotting the measured one-step entropy change against $A_y$ at real training positions. Rather than exhibiting a single clear relationship, the points form a broad two-sided fan: updates with similar advantages can produce entropy changes of opposite signs and substantially different magnitudes.

Figure~\ref{fig:advantage_entropy_intuition}(\subref{fig:entropy_intuition}) provides the underlying intuition. Consider two updates with the same positive advantage $A_y>0$. Reinforcing a token that already has high probability further concentrates the distribution and decreases entropy, whereas reinforcing a low-probability token can spread probability mass more evenly and increase entropy. Thus, the entropy effect is jointly determined by $A_y$ and the student’s current distribution, motivating a more systematic analysis of their interaction.

\begin{figure}[t]
    \centering

    \begin{subfigure}[t]{\linewidth}
    \centering

    \includegraphics[width=\linewidth]{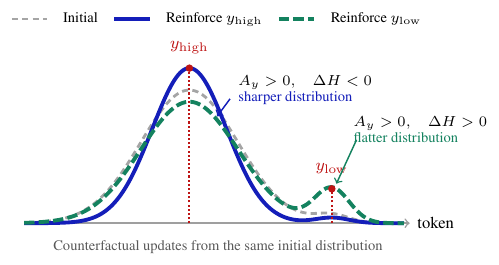}

    \caption{
        The same positive advantage can produce opposite entropy
        effects depending on the student's local probability structure.
    }
    \label{fig:entropy_intuition}
\end{subfigure}

    \vspace{1.5mm}

    \begin{subfigure}[t]{0.49\linewidth}
        \centering

        \includegraphics[
            width=\linewidth
        ]{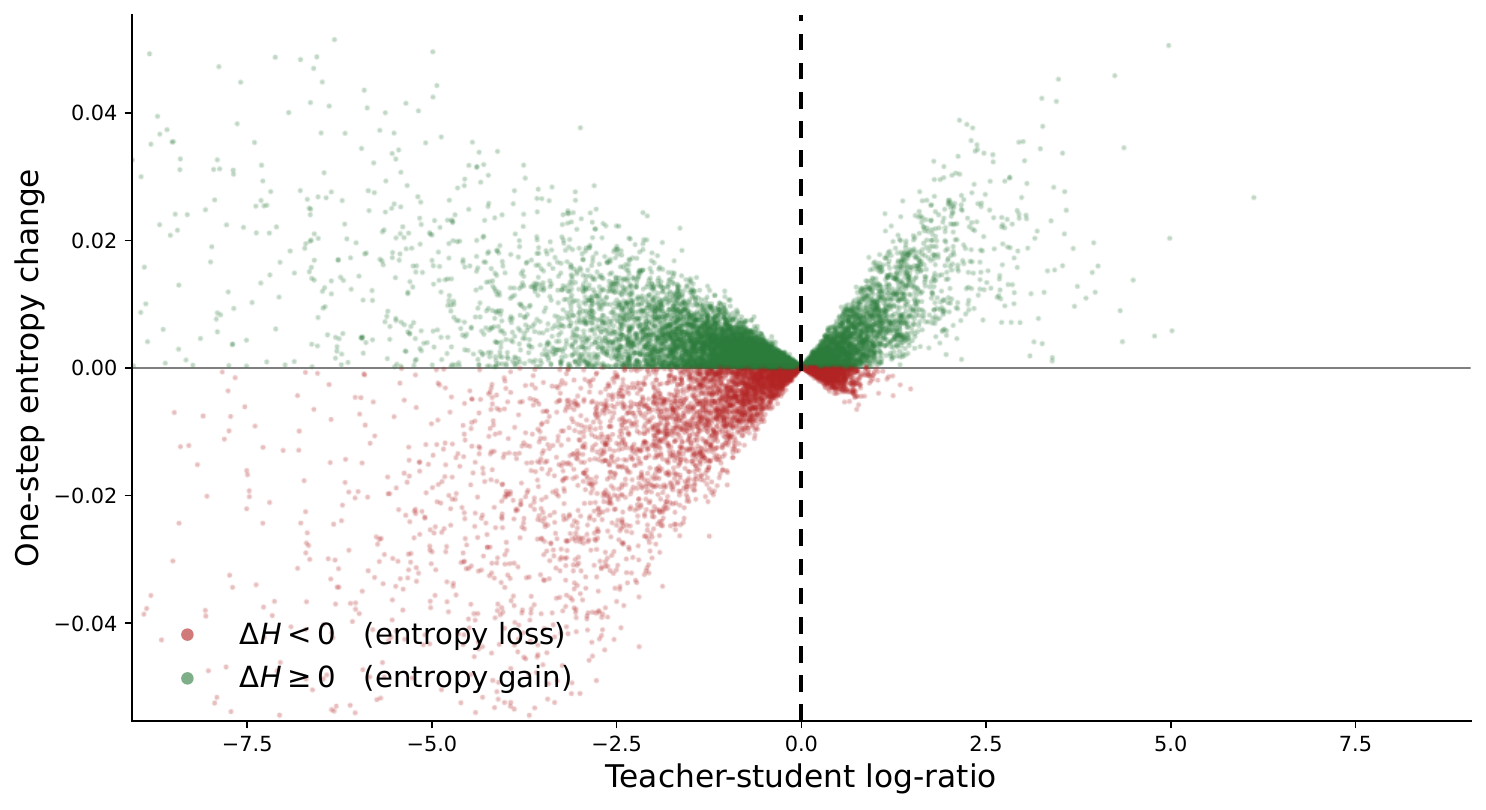}

        \caption{
            Advantage $A_y$ versus measured entropy change.
        }
        \label{fig:advantage_vs_entropy}
    \end{subfigure}
    \hfill
    \begin{subfigure}[t]{0.49\linewidth}
        \centering

        \includegraphics[
            width=\linewidth
        ]{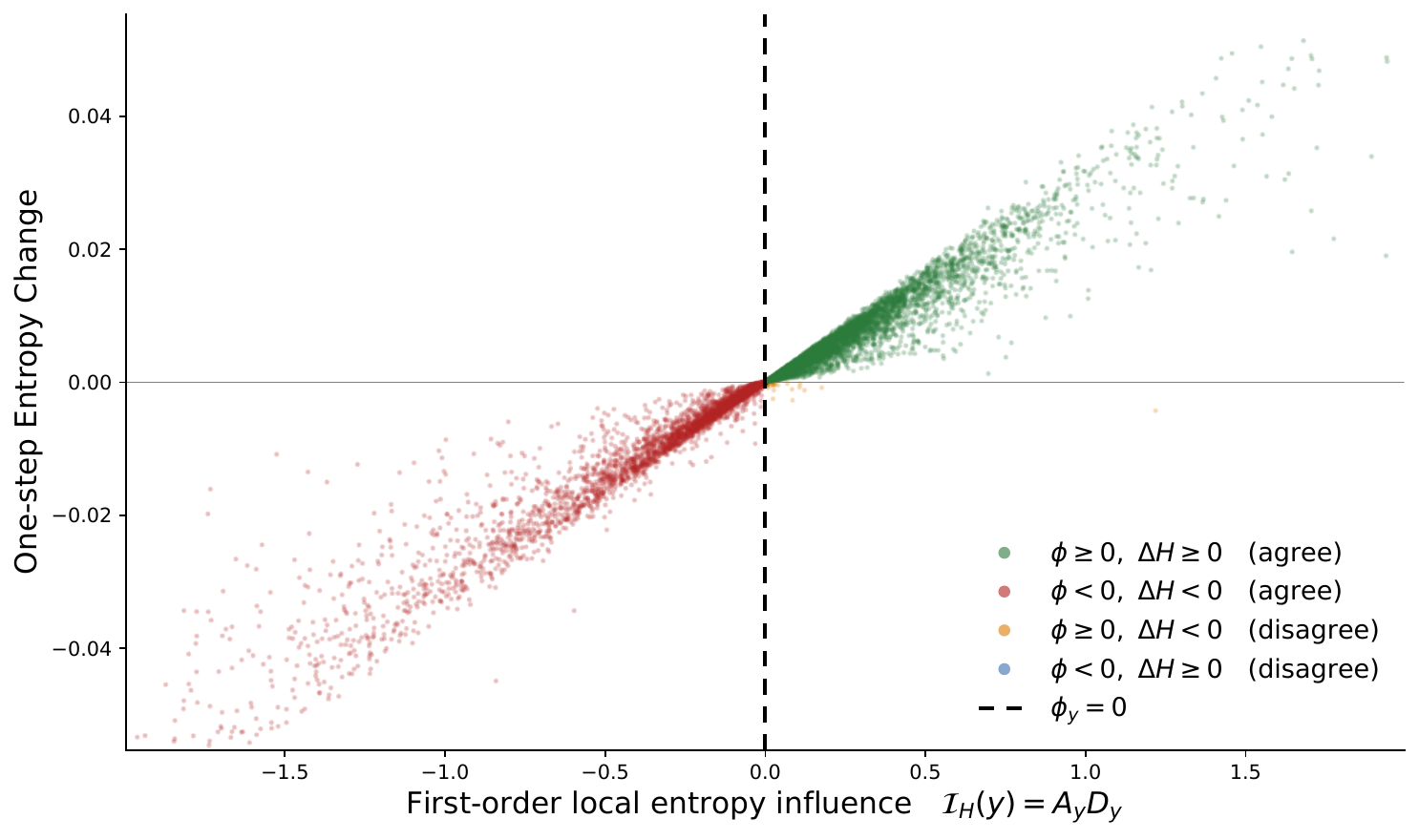}

        \caption{
            First-order influence $\mathcal{I}_H(y)$ versus
            measured entropy change.
        }
        \label{fig:influence_vs_entropy}
    \end{subfigure}

    \caption{
        \textbf{Why advantage alone does not determine the entropy
        effect of a sampled-token OPD update.}
        (a) An illustration: the same positive
        advantage sharpens the distribution when it reinforces a
        high-probability mode (lowering entropy) but flattens it when
        it promotes a low-probability alternative (raising entropy).
        (b) On real training positions,
        the advantage $A_y$ correlates weakly with the measured entropy
        change: the same $A_y$ spans both signs.
        (c) The first-order influence
        $\mathcal{I}_H(y)$ correlates tightly with the measured entropy
        change and predicts its sign.
    }
    \label{fig:advantage_entropy_intuition}
\end{figure}

\subsection{Diagnostic: First-Order Local Entropy Influence}
\label{sec:entropy_influence}

Let $p=(p_i)_{i\in\Vocab}$ be the student's current next-token distribution, where $p_i=p_\theta(i\mid h)$, with entropy $H(p)=-\sum_{i\in\Vocab}p_i\log p_i$. Let $p'$ denote the distribution after a small local sampled-token OPD update with step size $\eta$, and define the resulting local entropy change as
\begin{equation*}
    \Delta H = H(p')-H(p).
\end{equation*}
Our goal is to understand how the sampled-token advantage $A_y$ and the student's current distribution jointly determine the sign and magnitude of $\Delta H$. We characterize this interaction through a first-order expansion.

\begin{theorem}[First-Order Local Entropy Influence]
\label{thm:entropy_influence}
Under a local logit-space gradient step on the sampled-token OPD loss with learning rate $\eta$, the local entropy change satisfies
\begin{equation*}
    \Delta H
    =
    \eta \cdot \mathcal I_H(y)+\mathcal O(\eta^2),
\end{equation*}
where we define the \emph{First-Order Local Entropy Influence} as
$\mathcal I_H(y)=A_yD_y$. The entropy-direction factor is
\begin{equation}
    D_y=\sum_i p_i^2\bigl(\log p_i+H(p)\bigr)
    -p_y\bigl(\log p_y+H(p)\bigr).
    \label{eq:dy}
\end{equation}
\end{theorem}

\begin{proof}[Proof sketch]
A single sampled-token OPD update perturbs the logits along
$\Delta z_k = \eta A_y(\I[k=y]-p_k)$, so the advantage $A_y$ scales the
{entire} update while its direction $(\I[k=y]-p_k)$ is fixed by the
student. Propagating this through the softmax Jacobian
$\partial p_i/\partial z_k = p_i(\I[i=k]-p_k)$ gives a first-order
probability change in which $A_y$ again factors out cleanly,
\begin{equation*}
    \Delta p_i
    =
    p_i'-p_i
    =
    \eta A_y\, p_i\bigl(\I[i=y]-p_i-p_y+S_2\bigr)
    +\mathcal O(\eta^2),
\end{equation*}
where $S_2=\sum_k p_k^2$.
Expanding the entropy and using $\sum_i\Delta p_i=0$ gives
$\Delta H=-\sum_i\log p_i\,\Delta p_i+\mathcal O(\eta^2)$.
Substituting the probability change and collecting terms, the common
factor $\eta A_y$ pulls out, leaving a purely student-side remainder
$D_y$:
\begin{equation*}
    \Delta H
    = \eta\,A_y D_y+\mathcal O(\eta^2).
\end{equation*}
The full derivation is given in Appendix.
\end{proof}

This decomposition is the crux of our diagnostic: to first order, the local entropy change $\Delta H$ is governed by the product of two components, the advantage $A_y=\log q_y-\log p_y$ and the factor $D_y$, which is determined entirely by the student's current probability distribution. This formalizes the observation in Section~\ref{subsec:advantage_entropy} that the entropy effect is jointly determined by $A_y$ and $p$, and Figure~\ref{fig:advantage_entropy_intuition}(\subref{fig:influence_vs_entropy}) indicates it on real training positions: unlike the two-sided pattern produced by $A_y$, $\mathcal I_H(y)$ closely tracks the measured one-step entropy change. Governing the leading-order term, $\mathcal I_H(y)$ thus provides a natural token-wise signal for designing transformations that selectively modify entropy-contracting updates while preserving entropy-expanding ones to protect the student's entropy during distillation.

\section{Influence-Directed Adaptive On-Policy Distillation}
\label{sec:method}

As established in Section~\ref{sec:understanding_failure}, the First-Order Local Entropy Influence $\mathcal{I}_H(y)$ flags entropy-contracting updates, but uniformly penalizing all of them risks discarding critical teacher corrections. To design a precise intervention, we first trace the empirical origin of the entropy loss. Selective OPD commonly assumes the most consequential entropy shifts occur at high-divergence tokens where teacher and student sharply disagree~\citep{xu2026tip,wang2026teachability,jin2026entropy}. To test this, we bin sampled-token updates by the normalized discrepancy $\delta_y = \frac{q_h(y)-p_h(y)}{q_h(y)+p_h(y)} \in [-1, 1]$ and, within each bin, measure the cumulative entropy loss, the actual shift in policy entropy across genuine optimizer steps, not the first-order proxy, alongside the token count.

\begin{figure}[t]
    \centering
    \includegraphics[width=0.85\columnwidth]{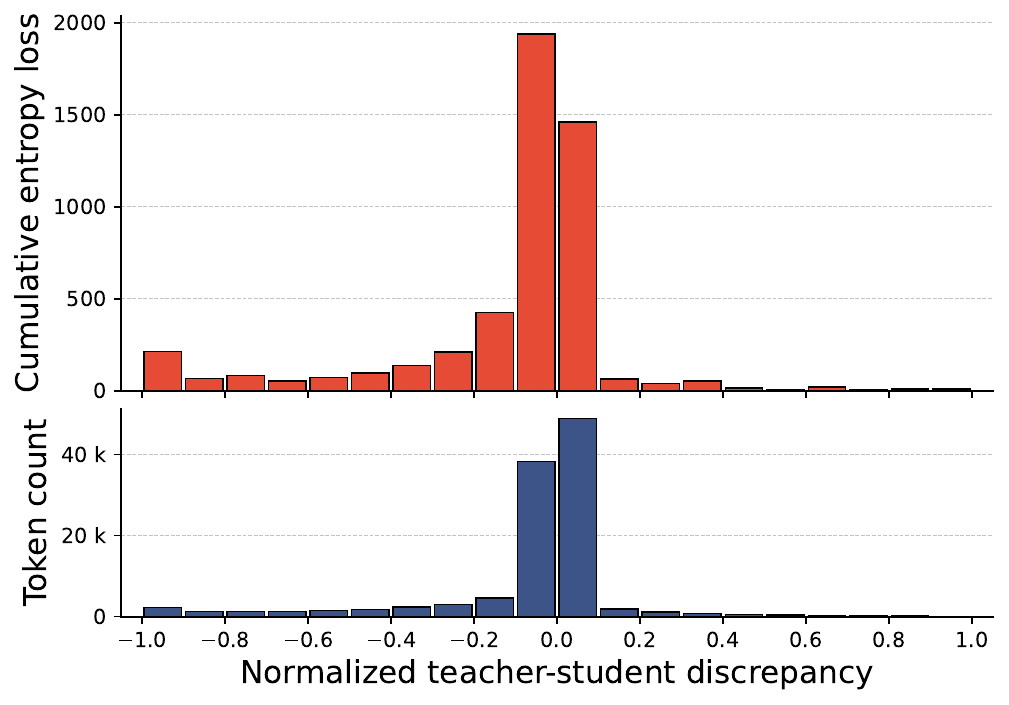}
    \caption{
    Distribution of cumulative entropy loss (top) and token counts (bottom) across normalized teacher--student discrepancy $\delta_y$. While high-divergence updates do contribute to entropy shifts ($\delta_y \approx -1.0$), the contraction overwhelmingly peaks at $\delta_y \approx 0$. This dominant entropy drain is driven by a massive, previously overlooked volume of highly aligned tokens.
    }
    \label{fig:entropy_flux}
\end{figure}

While this assumption captures part of the picture, Figure~\ref{fig:entropy_flux} indeed reveals a secondary cluster of entropy loss at the high-divergence negative tail ($\delta_y \approx -1.0$), it overlooks the primary driver of entropy collapse. As shown in the figure, the cumulative contraction overwhelmingly peaks in the low-discrepancy region ($\delta_y \approx 0$). Although each of these highly aligned updates has a minimal individual impact on entropy, the token-count histogram (bottom) exposes a massive population of such tokens, which in aggregate accumulate into the dominant source of the entropy drain. Consequently, a uniform penalty on all entropy-contracting updates ($\mathcal{I}_H(y) < 0$) that ignores this discrepancy is suboptimal: it suppresses critical teacher signals at the high-divergence tail while mistreating the massive density-driven loss at the center. The intervention must instead scale adaptively with $\delta_y$, motivating the divergence-adaptive shrinkage below.

\subsection{Solution: Divergence-Adaptive Shrinkage}
\label{subsec:adaptive_shrinkage}

We therefore propose \emph{Influence-Directed Adaptive On-Policy Distillation} (IDA-OPD), which attenuates the advantage $A_y$ at entropy-contracting positions ($\mathcal{I}_H(y)<0$) in proportion to how much correction the update still warrants. We measure this by the symmetric relative teacher--student disagreement at the sampled token,
\begin{equation}
    w_y = \frac{|q_y-p_y|}{q_y+p_y} \in [0,1).
\end{equation}
Acting as a scale-free retention factor, $w_y$ is small when teacher and student already agree (little correction needed) and approaches $1$ as their disagreement grows; it introduces no signal beyond $A_y$.

We use $w_y$ to modulate the advantage at entropy-contracting positions:
\begin{align}
    \widetilde A_y
    &=
    \begin{cases}
        A_y, & \mathcal I_H(y)\ge 0, \\
        w_y\,A_y, & \mathcal I_H(y)<0,
    \end{cases}
    \\
    \ell_y^{\mathrm{IDA\text{-}OPD}}
    &= -\sg(\widetilde A_y)\log p_y.
\end{align}
By construction, this rule is {sign-preserving}, it never flips the direction of the teacher correction, and {divergence-adaptive}: it multiplies the advantage by a weight that vanishes exactly where further optimization would merely drain entropy while leaving high-disagreement corrections intact.

\begin{proposition}[Quadratic attenuation near agreement]
\label{prop:o_a2}
For every prefix $h$ and sampled token $y$, $w_y\,A_y = \mathcal O\!\left(A_y^{2}\right)$ as $A_y\to 0$, and $w_y\,A_y\to A_y$ as $|A_y|\to\infty$.
\end{proposition}

The proof, together with the closed-form identity relating $w_y$ to $A_y$, is given in Appendix.
Proposition~\ref{prop:o_a2} formalizes the two
behaviors required from a divergence-adaptive shrinkage: near-quadratic
attenuation of low-discrepancy updates that contribute most of the
aggregate entropy drain and near-lossless preservation of high-divergence
corrections that carry substantive teacher signals. Consequently, IDA-OPD mitigates diversity
distillation failure without invoking full-vocabulary Forward KL
evaluations.

\section{Experiments}
\paragraph{Models.}
We evaluate IDA-OPD across two domains. For mathematics we use two settings spanning different scales: (1) Qwen3-8B-Non-Thinking-RL-Math $\rightarrow$ Qwen3-8B-Non-Thinking and (2) Qwen3-4B-Non-Thinking-RL-Math $\rightarrow$ Qwen3-4B-Non-Thinking; As for code, we use Qwen3-4B-Non-Thinking-RLCode $\rightarrow$ Qwen3-4B-Non-Thinking. All teachers are GRPO-trained following \cite{yang2026learning}.

\paragraph{Training Details.}
For mathematics we follow \cite{li2026filter}, training on DeepMath103K filtered to difficulty level 6 (the hardest problems). For code we distill from Qwen3-4B-Non-Thinking-RLCode~\cite{yang2026learning}, training on the code subset filtered from its data~\cite{yang2026learning}.



\paragraph{Evaluation.}
For mathematics, we evaluate on AIME 2024, AIME 2025, HMMT 2025 Feb, and HMMT 2025 Nov. Following \citet{zhu2026surprising}, we adopt the unbiased pass@$k$ estimator introduced by \citet{chen2021evaluating}. Specifically, for each problem $x$, we generate $n=128$ responses using independent random seeds, with temperature $1.0$ and top-$p=1.0$. Let $c_x$ denote the number of correct responses among the $n$ samples. We estimate pass@$k$ as
\begin{equation}
    \mathrm{pass@}k
    =
    \mathbb{E}_{x\sim\mathcal D}
    \left[
        1-
        \frac{\binom{n-c_x}{k}}
             {\binom{n}{k}}
    \right],
    \qquad n\ge k.
\end{equation}
Compared with directly evaluating pass@$k$ using only $k$ samples per problem, generating $n>k$ samples and applying the unbiased estimator mitigates the variance caused by stochastic decoding. We report pass@1 and pass@16 in Table~\ref{tab:main_results}, with the complete pass@$k$ curves presented in Section~\ref{subsec:passk}.

\paragraph{Baselines.}
We compare against standard OPD~\cite{lu2025opd}, three recent improvements (EOPD~\cite{jin2026entropy}, AOPD~\cite{jia2026aopd}, and REOPOLD~\cite{ko2026reopold}), and two entropy-based exploration strategies on top of OPD: an entropy bonus~\cite{schulman2017proximal} and advantage shaping~\cite{cheng2025reasoning}. EOPD, AOPD and REOPOLD use the official implementation.

\begin{table*}[!ht]
\centering
\small
\renewcommand{\arraystretch}{1.05}
\setlength{\tabcolsep}{7pt}
\begin{tabular}{lcccccccc}
\toprule
\multirow{2}{*}{\textbf{Method}}
& \multicolumn{2}{c}{\textbf{AIME24}}
& \multicolumn{2}{c}{\textbf{AIME25}}
& \multicolumn{2}{c}{\textbf{HMMT Feb}}
& \multicolumn{2}{c}{\textbf{HMMT Nov}} \\
\cmidrule(lr){2-3}
\cmidrule(lr){4-5}
\cmidrule(lr){6-7}
\cmidrule(lr){8-9}
& pass@1 & pass@16
& pass@1 & pass@16
& pass@1 & pass@16
& pass@1 & pass@16 \\
\midrule

\multicolumn{9}{c}{\textit{Qwen3-8B-Non-Thinking-RL-Math $\rightarrow$ Qwen3-8B-Non-Thinking}} \\
\midrule
Student             & 33.3 & 56.7 & 30.0 & 43.3 & 6.7 & 26.7 & 23.3 & 33.3 \\
Teacher             & 60.0 & 83.3 & 53.5 & 73.3 & 30.0 & 50.0 & 46.7 & 63.3 \\
\midrule
+ OPD               & 61.7 & 79.1 & 47.9 & 70.7 & 37.5 & 49.5 & 38.6 & 60.8 \\
+ REOPOLD           & 64.0 & 77.4 & 48.5 & 72.6 & 36.5 & 49.6 & 38.9 & 61.3 \\
+ Entropy Bonus     & 62.0 & 78.3 & 47.2 & 72.0 & 36.5 & 50.8 & 38.4 & 61.7 \\
+ Advantage Shaping & 63.2 & 79.6 & 47.8 & 72.4 & 37.0 & 52.6 & 39.0 & 62.3 \\
+ AOPD              & \textbf{65.9} & 81.0 & 48.6 & 76.0 & 37.2 & 54.4 & 37.6 & 62.7 \\
+ EOPD              & 65.6 & 80.8 & 48.0 & 72.6 & 37.5 & 54.1 & 39.8 & 62.8 \\
\midrule
Ours                & 63.3 & \textbf{83.3} & \textbf{50.0} & \textbf{76.7}
                    & 37.6 & \textbf{56.7} & \textbf{40.1} & \textbf{63.3} \\
$\Delta$ vs OPD     & \gain{1.6} & \gain{4.2} & \gain{2.1} & \gain{6.0}
                    & \gain{0.1} & \gain{7.2} & \gain{1.5} & \gain{2.5} \\
\midrule

\multicolumn{9}{c}{\textit{Qwen3-4B-Non-Thinking-RL-Math $\rightarrow$ Qwen3-4B-Non-Thinking}} \\
\midrule
Student             & 21.7 & 56.7 & 22.5 & 56.7 & 23.1 & 39.0 & 11.2 & 36.7 \\
Teacher             & 53.3 & 80.0 & 53.5 & 73.3 & 26.7 & 53.5 & 35.4 & 63.3 \\
\midrule
+ OPD               & 54.6 & 78.7 & 51.8 & 65.7 & \textbf{26.5} & 51.7 & 31.2 & 58.3 \\
+ REOPOLD           & 57.2 & 77.6 & 52.1 & 66.1 & 25.3 & 52.1 & 31.6 & 58.7 \\
+ Entropy Bonus     & 55.4 & 78.2 & 51.4 & 66.3 & 25.6 & 53.5 & 31.0 & 58.9 \\
+ Advantage Shaping & 56.8 & 79.4 & 52.0 & 66.4 & 26.0 & 55.6 & 31.5 & 59.5 \\
+ AOPD              & \textbf{61.0} & 81.1 & \textbf{52.4} & 69.7 & 26.1 & 57.4 & 30.3 & 59.4 \\
+ EOPD              & 59.2 & 80.0 & 51.2 & 66.6 & 26.1 & 57.3 & 31.9 & 59.8 \\
\midrule
Ours                & 56.7 & \textbf{83.3} & 53.3 & \textbf{70.0}
                    & 26.3 & \textbf{60.0} & \textbf{32.3} & \textbf{60.1} \\
$\Delta$ vs OPD     & \gain{2.1} & \gain{4.6} & \gain{1.5} & \gain{4.3}
                    & \drop{0.2} & \gain{8.3} & \gain{1.1} & \gain{1.8} \\
\bottomrule
\end{tabular}
\caption{Distillation results across multiple datasets and models. Best results are highlighted in \textbf{bold}. We report pass@1 and pass@16 on four competition-mathematics benchmarks under two teacher$\rightarrow$student settings . The $\Delta$ vs OPD rows report our absolute gains over vanilla OPD, showing that IDA-OPD consistently lifts pass@16 while broadly maintaining pass@1.}
\label{tab:main_results}
\end{table*}

\subsection{Main Results}

Table~\ref{tab:main_results} reports pass@1 and pass@16 across two model scales. Student-side methods (e.g., OPD, REOPOLD) exhibit clear diversity distillation failure: while they improve pass@1, their pass@16 severely trails the teacher. For instance, in the 4B setting, OPD's pass@16 stagnates at $78.7\%$ vs.\ the teacher's $80.0\%$ on AIME 2024, and $65.7\%$ vs.\ $73.3\%$ on AIME 2025. A similar entropy collapse occurs in the 8B setting.

IDA-OPD successfully mitigates this failure, attaining the highest pass@16 across all four benchmarks at both scales. On the 4B setting, it improves over standard OPD by \gain{8.3} (HMMT Feb), \gain{4.6} (AIME 2024), and \gain{4.3} (AIME 2025). The gains on the 8B setting are equally substantial, yielding improvements of \gain{7.2}, \gain{4.2}, and \gain{6.0}, respectively. Crucially, IDA-OPD matches or explicitly surpasses the teacher's pass@16 on several benchmarks (e.g., achieving $83.3\%$ vs.\ $80.0\%$ on 4B AIME 2024, and $76.7\%$ vs.\ $73.3\%$ on 8B AIME 2025). Furthermore, it achieves this diversity recovery while broadly \emph{improving} pass@1 over OPD.

When compared to other interventions, naive student-side adjustments like the entropy bonus~\citep{schulman2017proximal} and advantage shaping~\citep{cheng2025reasoning} fail to match IDA-OPD's pass@16, as inflating entropy indiscriminately cannot disentangle near-deterministic sharpening from informative corrections. 

Meanwhile, the teacher-informed AOPD and EOPD recover a large portion of the diversity, but IDA-OPD matches or consistently exceeds their performance (achieving the highest pass@16 overall in Table~\ref{tab:main_results}). The key distinction is that IDA-OPD accomplishes this strong diversity preservation in the evaluated settings using only a teacher-free entropy diagnostic, strictly avoiding the expensive top-$K$ teacher distributions required by those baselines.

Beyond mathematics, Table~\ref{tab:performance_on_code} shows the same pattern on the code domain: OPD leaves pass@16 flat ($72.9$ vs.\ the student's $73.5$ on MBPP+), while IDA-OPD improves both metrics, \gain{2.1}/\gain{0.9} on MBPP+ and \gain{1.3}/\gain{1.0} on LiveCodeBench (pass@1/pass@16), suggesting that our diagnosis and intervention transfer across different domains.

\subsection{Pass@$k$ Performance}
\label{subsec:passk}

Following prior work that measures diversity preservation via test-time scaling~\cite{cheng2025reasoning,jin2026entropy,ko2026reopold}, we use $\mathrm{pass@}k$ as a diversity-sensitive measure of the model's ability to find at least one correct solution across repeated samples. Figure~\ref{fig:passk_curves} reports pass@$k$ on AIME 2024 and AIME 2025 for $k=1$ to $64$. IDA-OPD is consistently higher across the range: the two methods are close at $k=1$ ($\sim\!+2$ points, where a single high-probability mode suffices), but the gap widens to $+4$ to $+5$ points by $k=8$--$16$, precisely where OPD narrows its output distribution to a few dominant modes while IDA-OPD retains enough diversity to keep discovering correct solutions. This advantage is sustained through $k=64$.

\begin{figure}[t]
\centering
\includegraphics[width=\columnwidth]{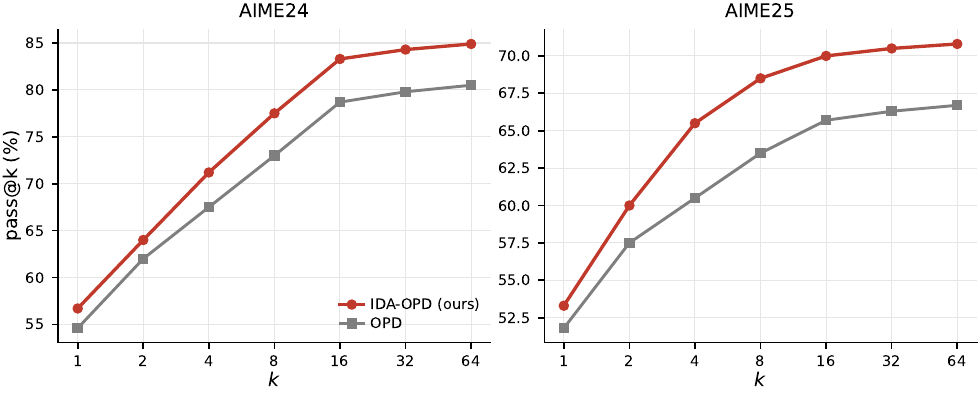}
\caption{pass@$k$ on AIME 2024 and AIME 2025 ($k=1$ to $64$) for IDA-OPD
versus standard OPD. The two methods are close for small $k$, but the
gap widens steadily as $k$ increases, reflecting the greater trajectory
diversity retained by IDA-OPD.}
\label{fig:passk_curves}
\end{figure}

\begin{table}[tb]
\centering
\small 
\setlength{\tabcolsep}{7pt} 
\begin{tabular}{l c c c c}
\toprule
\multirow{2}{*}{\textbf{Method}} & \multicolumn{2}{c}{\textbf{MBPP+}} & \multicolumn{2}{c}{\textbf{LCB}} \\
\cmidrule(lr){2-3} \cmidrule(lr){4-5}
& pass@1 & pass@16 & pass@1 & pass@16 \\
\midrule
Student          & 64.6 & 73.5 & 17.5 & 50.4 \\
OPD              & 69.5 & 72.9 & 26.8 & 54.2 \\
+ EOPD           & 70.3 & 73.7 & \textbf{29.2} & 54.9 \\
\midrule
\textbf{IDA-OPD}  & \textbf{71.6} & 73.8 & 28.1 & 55.2 \\
$\Delta$ vs OPD  & +2.1 & +0.9 & +1.3 & +1.0 \\
\bottomrule
\end{tabular}
\caption{Distillation results on the code domain (pass@1/pass@16 on MBPP+ and LiveCodeBench) for Qwen3-4B-Non-Thinking-RLCode $\rightarrow$ Qwen3-4B-Non-Thinking. Best results in \textbf{bold}.}
\label{tab:performance_on_code}
\end{table}

\subsection{Ablation Study}
\label{subsec:ablation}

IDA-OPD is composed of two coupled design choices: (i) using the sign of $\mathcal I_H(y)$ to select which positions to intervene on, and (ii) replacing the advantage at those positions with the divergence-adaptive shrinkage $w_y A_y$. To isolate the contribution of each component, we design three ablations, all trained under the same setting as IDA-OPD on the Qwen3-4B-Non-Thinking-RL-Math $\rightarrow$ Qwen3-4B block:
\begin{itemize}
    \item \textbf{Uniform shrinkage (w/o $\mathcal I_H$ gate).} Apply the shrinkage at every position, $\widetilde A_y=w_y A_y$, ignoring the sign of $\mathcal I_H(y)$.
    \item \textbf{Hard mask (w/o shrinkage).} Keep the gate but zero out entropy-contracting updates: $\widetilde A_y = A_y$ if $\mathcal I_H(y)\ge 0$, else $\widetilde A_y = 0$.
    \item \textbf{$\operatorname{sign}(A_y)$ gate.} Replace the gate with the advantage sign ($w_y A_y$ when $A_y<0$), testing whether $\mathcal I_H(y)$ adds information beyond the sign of $A_y$.
\end{itemize}


\begin{table}[tb]
\centering
\small 
\setlength{\tabcolsep}{1pt} 
\begin{tabular}{l c c c c}
\toprule
\multirow{2}{*}{\textbf{Variant}} & \multicolumn{2}{c}{\textbf{AIME24}} & \multicolumn{2}{c}{\textbf{AIME25}} \\
\cmidrule(lr){2-3} \cmidrule(lr){4-5} 
& pass@1 & pass@16 & pass@1 & pass@16 \\ 
\midrule
IDA-OPD (full)            & \textbf{56.7} & \textbf{83.3} & \textbf{53.3} & \textbf{70.0} \\
\midrule
w/o $\mathcal{I}_H$ gate  & 53.8 & 81.5 & 51.2 & 68.4 \\
w/o shrinkage (hard mask) & 52.1 & 82.0 & 50.4 & 68.9 \\
sign($A_y$) gate          & 54.2 & 76.8 & 51.9 & 65.2 \\
\bottomrule
\end{tabular}
\caption{Ablation study on the Qwen3-4B-Non-Thinking-RL-Math $\rightarrow$ Qwen3-4B setting. Each row disables one design choice of IDA-OPD while keeping the rest of the training pipeline unchanged.}
\label{tab:ablation}
\end{table}

Table~\ref{tab:ablation} shows every ablation degrades relative to full IDA-OPD. Removing the $\mathcal I_H(y)$ gate and shrinking indiscriminately over-attenuates entropy-expanding updates, dropping pass@1 ($56.7\to53.8$ on AIME24) with only modest pass@16 loss, so entropy-directed selection is necessary. Hard masking preserves diversity reasonably (pass@16 $82.0$/$68.9$) but inflicts the largest pass@1 damage ($56.7\to52.1$), discarding teacher corrections whose entropy cost is small relative to their value. Using $\operatorname{sign}(A_y)$ as the gate misidentifies entropy-risk positions, the sign of $A_y$ does not determine the entropy-change sign (Section~\ref{subsec:advantage_entropy}), so it protects diversity worst (pass@16 $76.8$/$65.2$) while trailing on pass@1. The gate and the shrinkage are thus complementary and both required.

\subsection{Effect of the Divergence-Adaptive Shrinkage}
\label{subsec:shrinkage_shape}

We now examine the {shape of the shrinkage}, i.e., how entropy-contracting updates are attenuated as a function of disagreement. The shrinkage sets the multiplier $c_y$ from the symmetric relative disagreement $w_y=\tfrac{|q_y-p_y|}{q_y+p_y}\in[0,1)$. Holding the $\mathcal I_H(y)$ gate fixed, we compare four shapes: a non-adaptive \textbf{Constant}, concave \textbf{Sqrt} ($c_y=w_y^{1/2}$), our linear \textbf{Ours} ($c_y=w_y$), and convex \textbf{Square} ($c_y=w_y^{2}$).

\begin{table}[t]
\centering
\renewcommand{\arraystretch}{1.15}
\small
\setlength{\tabcolsep}{1pt}
\begin{tabular}{lcccc}
\toprule
\multirow{2}{*}{\textbf{Shrinkage $c_y=f(w_y)$}} & \multicolumn{2}{c}{\textbf{AIME24}} & \multicolumn{2}{c}{\textbf{AIME25}} \\
\cmidrule(lr){2-3} \cmidrule(lr){4-5}
 & pass@1 & pass@16 & pass@1 & pass@16 \\
\midrule
Constant ($c_y=\alpha$)      & 54.6 & 79.2 & 51.5 & 66.7 \\
Sqrt ($c_y=w_y^{1/2}$)       & 55.9 & 80.4 & 52.4 & 67.5 \\
Ours ($c_y=w_y$)             & \textbf{56.7} & \textbf{83.3} & \textbf{53.3} & \textbf{70.0} \\
Square ($c_y=w_y^{2}$)       & 54.1 & 81.6 & 51.0 & 68.8 \\
\bottomrule
\end{tabular}
\caption{Effect of the disagreement-to-shrinkage mapping $f(w_y)$ on the Qwen3-4B-Non-Thinking-RL-Math $\rightarrow$ Qwen3-4B setting. All variants share the same $\mathcal I_H(y)$ gate.}
\label{tab:shrinkage_shapes}
\end{table}

Table~\ref{tab:shrinkage_shapes} shows our linear map tops every column ($56.7$/$83.3$, $53.3$/$70.0$). The Constant map is worst, leaving the low-divergence entropy drain uncurbed while weakening high-divergence corrections; concave Sqrt under-attenuates low-divergence updates (pass@16 $80.4$/$67.5$); convex Square over-suppresses moderate-divergence corrections that still carry teacher signal, posting the lowest pass@1. Scaling attenuation linearly with $w_y$ thus best balances suppressing low-discrepancy drain against preserving high-discrepancy corrections.

\subsection{Validating Entropy-Influence under Training Dynamics}
\label{subsec:training_dynamics}


We now isolate the effect of $\mathcal{I}_H(y)$-based selection to confirm that it successfully targets entropy-draining updates in practice. Figure~\ref{fig:entropy_dynamics} tracks the student's policy entropy during training under vanilla OPD and two hard-masking variants: one that drops all entropy-contracting updates ($\mathcal{I}_H(y)<0$), and another that drops all entropy-expanding updates ($\mathcal{I}_H(y)\ge0$). 

The resulting trajectories cleanly bracket the baseline, supporting that $\mathcal{I}_H(y)$ is an effective control signal. Crucially, while Theorem 1 assumes idealized logit-space, single-step gradient updates, these global trajectory shifts indicate that $\mathcal{I}_H(y)$ remains a robust directional proxy under actual parameter-space optimization with shared weights and AdamW dynamics. This empirical evidence bridges the gap between our localized theoretical derivation and full-scale training.

Masking the expanding updates drives entropy far below vanilla OPD, whereas masking the contracting updates lifts it to the highest overall level. This indicates that selecting updates by the sign of $\mathcal{I}_H(y)$ steers the model's global entropy trajectory. However, maximizing entropy via a hard mask is over-protective. Discarding {every} entropy-contracting update sacrifices crucial knowledge signals whose knowledge value outweighs their minor entropy cost, which is why this variant suffers the worst pass@1 drop in Table~\ref{tab:ablation}.

\begin{figure}[t]
\centering
\begin{subfigure}[t]{0.49\columnwidth}
\centering
\includegraphics[width=\linewidth]{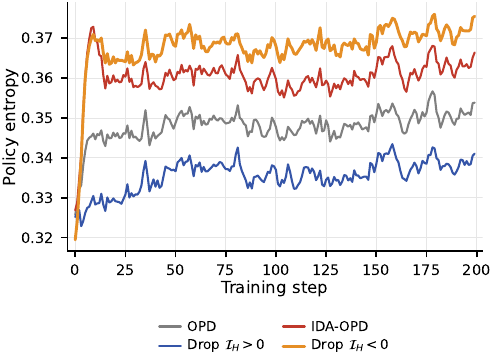}
\caption{}
\label{fig:entropy_dynamics}
\end{subfigure}
\hfill
\begin{subfigure}[t]{0.49\columnwidth}
\centering
\includegraphics[width=\linewidth]{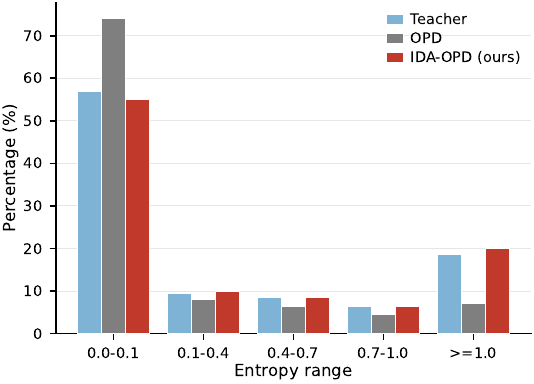}
\caption{}
\label{fig:entropy_hist}
\end{subfigure}
\caption{Entropy analysis of IDA-OPD. \textbf{(a)} Policy entropy over training: masking entropy-expanding updates ($\mathcal I_H\!<\!0$ kept) drives entropy below vanilla OPD, masking entropy-contracting updates ($\mathcal I_H\!\ge\!0$ kept) lifts it highest but over-protects, and IDA-OPD reweights softly to sit between the two while staying well above OPD. \textbf{(b)} Token-level entropy on AIME 2024 problems: OPD piles far more mass in the lowest bin ($0.0$--$0.1$) and under-represents high-entropy tokens ($\geq\!1.0$), whereas IDA-OPD stays close to the teacher at both ends.}
\label{fig:entropy_analysis}
\end{figure}

\subsection{Token-Level Entropy Analysis}
\label{subsec:token_entropy}

To examine how IDA-OPD reshapes per-token uncertainty, we run a token-level analysis on AIME 2024 dataset. Each model rolls out its own response and, following Section~\ref{subsec:advantage_entropy}, we compute the entropy of its predicted distribution at every token; Figure~\ref{fig:entropy_hist} aggregates these into a histogram comparing the teacher against OPD- and IDA-OPD-trained students.

In the mid-entropy range ($\sim\!0.1$--$1.0$) OPD and IDA-OPD both track the teacher; the differences concentrate at the ends. In the lowest-entropy bin ($0.0$--$0.1$) OPD piles up far more mass than the teacher ($\sim\!74\%$ vs.\ $\sim\!57\%$), the signature of over-confident, near-deterministic predictions from entropy contraction, whereas IDA-OPD stays close ($\sim\!55\%$) and correspondingly retains more mass in the high-entropy region ($\geq\!1.0$) where OPD is depleted. This is the downstream trace of the low-$\delta_y$ pressure in Figure~\ref{fig:entropy_flux}: by attenuating updates in proportion to discrepancy, IDA-OPD avoids the $0.0$--$0.1$ pile-up and keeps the distribution closer to the teacher's.

\subsection{Computational Efficiency}
\label{subsec:efficiency}

Like OPD, IDA-OPD transfers only one scalar per position ($O(L)$), whereas teacher-informed methods (EOPD, AOPD) transfer a truncated top-$K$ teacher distribution at selected positions ($O(L_{\mathrm{sel}}K)$, $K\gg 1$); its only extra cost, the entropy term $D_y$, reuses the student logits already materialized in the forward pass. IDA-OPD thus keeps OPD's cheap footprint while matching the far more expensive teacher-informed baselines. The full per-rollout cost breakdown is given in Appendix.

\section{Related Work}

\paragraph{On-Policy Distillation for Language Models.} 
Offline knowledge distillation supervises the student on fixed teacher outputs~\cite{hinton2015distilling,kim2016sequence}, but never corrects it on the prefixes it visits at inference. On-policy distillation (OPD) removes this train--inference mismatch by training on the student's \emph{own} sampled prefixes~\cite{agarwal2024policy, yu2026preference}. OPD variants differ in per-token teacher cost: \emph{full-distribution} OPD matches the teacher's entire next-token distribution but needs full-vocabulary logits~\cite{gu2024minillm,ko2024distillm,xu2024speculative}, whereas \emph{sampled-token} OPD queries only the sampled token's probability, the single-sample estimator of the reverse KL~\cite{lu2025opd,fu2026revisiting,jia2026aopd}, and is far cheaper for long rollouts. This efficiency has extended OPD to cross-tokenizer, agent, and unified-framework settings~\cite{sun2026simct,li2026curriculum,sun2026easyopd,zhong2026sod}. Yet by relying on a single sampled-token signal, sampled-token OPD can quietly contract the student's diversity, an effect prior sampled-token work has not traced to individual updates.

\paragraph{Token Selection and Reweighting in OPD.}
The closest work reconsiders \emph{which} updates OPD should trust. Disagreement- and importance-based methods select or reweight tokens by teacher--student divergence: TIP by token importance~\cite{xu2026tip}, Token Teachability by whether a disagreement is learnable~\cite{wang2026teachability}, and Filter-Then-Reweight by optimization granularity~\cite{li2026filter}. AOPD decouples exploration from imitation~\cite{jia2026aopd}, while REOPOLD and negative-sample reinforcement reshape sampled-token credits~\cite{ko2026reopold,zhu2026surprising}. Most related, Entropy-Aware OPD (EOPD) restores lost entropy with a forward-KL term at high-uncertainty tokens~\cite{jin2026entropy}. All of these locate salient positions by teacher--student disagreement or uncertainty, and the strongest, EOPD~\cite{jin2026entropy} and AOPD~\cite{jia2026aopd}, reintroduce full or top-$K$ teacher distributions to act on them. IDA-OPD instead locates positions by the signed effect of each update on local entropy and shrinks only low-value, entropy-contracting updates, with no teacher query beyond the sampled-token probability. More broadly, RL studies show that policy entropy and a few high-entropy ``forking'' tokens govern reasoning diversity~\cite{cui2025entropy,wang2025beyond}; unlike that reward-driven setting, we ask how the OPD objective itself drains entropy and intervene at the level of individual updates.

\section{Conclusion}

To address diversity distillation failure in sampled-token OPD, we derived {First-Order Local Entropy Influence} ($\mathcal I_H(y)$), a cheap, teacher-free proxy that flags whether an update expands or contracts entropy. Guided by it, IDA-OPD preserves entropy-expanding updates and replaces entropy-contracting advantages with a divergence-adaptive shrinkage $w_y A_y$. It consistently improves pass@$k$ over methods sharing its sampled-token budget, reaching parity with the strongest teacher-informed methods at strictly lower cost while broadly maintaining vanilla OPD's pass@1.

\bibliography{aaai2027}

\end{document}